\documentclass[sn-mathphys-ay]{sn-jnl}

\usepackage{graphicx}%
\usepackage{multirow}%
\usepackage{amsmath,amssymb,amsfonts}%
\usepackage{amsthm}%
\usepackage{mathrsfs}%
\usepackage[title]{appendix}%
\usepackage{xcolor}%
\usepackage{textcomp}%
\usepackage{manyfoot}%
\usepackage{booktabs}%
\usepackage{algorithm}%
\usepackage{algorithmicx}%
\usepackage{algpseudocode}%
\usepackage{listings}%
\usepackage[authoryear]{natbib}%
\setcitestyle{aysep={}} %
\usepackage{subfig}%
\usepackage{amsthm} 
\usepackage{algorithm}
\newtheorem{lemma}{Lemma}  %

\theoremstyle{thmstyleone}%
\newtheorem{theorem}{Theorem}
\newtheorem{proposition}[theorem]{Proposition}%

\theoremstyle{thmstyletwo}%

\theoremstyle{thmstylethree}%

\begin{document}

\title[Article Title]{SPPCSO: Adaptive Penalized Estimation Method for High-Dimensional Correlated Data}

\author[1]{\fnm{Ying} \sur{Hu}}
\author*[1]{\fnm{Hu} \sur{Yang}}\email{yh@cqu.edu.cn}
\affil[1]{\orgdiv{College of Mathematics and Statistics}, \orgname{Chongqing University}, \orgaddress{\city{Chongqing}, \postcode{401331}, \country{China}}}


\abstract{ With the rise of high-dimensional correlated data, multicollinearity poses a significant challenge to model stability, often leading to unstable estimation and reduced predictive accuracy. This work proposes the Single-Parametric Principal Component Selection Operator (SPPCSO), an innovative penalized estimation method that integrates single-parametric principal component regression and $L_{1}$ regularization to adaptively adjust the shrinkage factor by incorporating principal component information. This approach achieves a balance between variable selection and coefficient estimation, ensuring model stability and robust estimation even in high-dimensional, high-noise environments. The primary contribution lies in addressing the instability of traditional variable selection methods when applied to high-noise, high-dimensional correlated data. Theoretically, our method exhibits selection consistency and achieves a smaller estimation error bound compared to traditional penalized estimation approaches. Extensive numerical experiments demonstrate that SPPCSO not only delivers stable and reliable estimation in high-noise settings but also accurately distinguishes signal variables from noise variables in group-effect structured data with highly correlated noise variables, effectively eliminating redundant variables and achieving more stable variable selection. Furthermore, SPPCSO successfully identifies disease-associated genes in gene expression data analysis, showcasing strong practical value. The results indicate that SPPCSO serves as an ideal tool for high-dimensional variable selection, offering an efficient and interpretable solution for modeling correlated data.
}




\maketitle

\section{Introduction}\label{sec1}

The linear model has remained a topic of sustained research interest in the field of statistics for decades and continues to be one of the most important and widely applied statistical methods. It takes the following form:
 \begin{eqnarray}\label{eq01}
 	y=X\beta+\varepsilon.
 \end{eqnarray}   
 where $y$ is a $n$ response vector, $X$ is a $n \times p$ design matrix and $\varepsilon$ is the error vector. We consider the components of the error vector to be independently distributed from $ N(\mathbf{0},\sigma^2\mathbf{I})$. $\beta$ is a vector of unknown regression coefficients. 
 
In low-dimensional settings (\( p \ll n \)), multicollinearity can make the design matrix \( X^{T}X \) ill-conditioned, leading to instability in ordinary least squares (OLS) estimation and large standard errors. To address this, classical biased estimators like ridge regression incorporate an \( L_2 \) penalty to reduce variance and improve stability. Another approach, principal component regression (PCR), exemplified by the single-parametric principal component regression estimator (\citet{Y1989}), integrates principal component analysis with adaptive shrinkage. It applies weaker shrinkage to important variables and stronger shrinkage to less influential ones, enhancing estimation accuracy and model interpretability.

 In high-dimensional settings, that is, when $n \ll p$ and $\beta$ have at most $q$ nonzero elements, variables often exhibit severe multicollinearity. Additionally, the high correlation makes it difficult to distinguish the independent contributions of each predictor variable to the response variable, thereby reducing the model's interpretability. Due to the high-dimensional nature of the data, the selection of the variables is required.

 Traditional methods of variable selection include stepwise selection and subset selection. Although these methods have utility, they often ignore the random errors inherited during the variable selection phase. \citet{T1996} proposed the Least absolute shrinkage and selection operator (Lasso) in the form of $\hat{\beta}:=\underset{\beta}{\text{argmin}}{\Vert y-X\beta\Vert_{2}^{2}+\lambda\Vert\beta\Vert_{1}}$, facilitating variable selection by zeroing out smaller parameters. \citet{ZY2006} systematically illustrated Lasso's consistency in model selection, showing that Lasso achieves consistency when an irrepresentability condition is satisfied. This condition requires that the non-zero coefficients in the true model are not excessively correlated with the zero coefficients. In addition, the penalty parameter must be appropriately chosen to balance between sparsity and model accuracy. Thus, penalized Lasso-type regularization methods in the form of
 \begin{eqnarray}\label{eq02}
 	\hat{\beta}:=\underset{\beta}{\text{argmin}}{\Vert y-X\beta\Vert_{2}^{2}+\sum_{j=1}^{p}p_\lambda(\Vert \beta_j \Vert_{1})},
 \end{eqnarray}
 have been widely studied and developed, where $p_\lambda(\cdot)$ is the penalty function. However, Lasso tends to select only one variable from a group of highly correlated variables. Therefore, in high-dimensional variable selection with strongly correlated predictors, although Lasso selects the minimal number of variables, this is a result of its tendency toward over-selection.
 
 \citet{FL2001} introduced the  SCAD penalty, i.e., $\dot{p}_\lambda (\Vert \beta_j\Vert_{1})=\lambda \{I(\Vert \beta_j \Vert_{1}\leq\lambda +\frac{(\gamma\lambda-\Vert \beta_j \Vert_{1})_+}{(\gamma-1)\lambda})I(\Vert \beta_j \Vert_{1} >\lambda)\}, \gamma>2$. The non-convex penalty permits both variable selection and parameter estimation, demonstrating oracle properties. \citet{Z2006}  proposed the Adaptive Lasso, with the penalty function $p_\lambda(\Vert \beta_j \Vert_{1})=\lambda\hat{v}_j\Vert\beta_j\Vert_{1}$, where $\hat{v}_j$ is a known weighting estimation. \citet{Z2010} proposed the MCP, i.e., $p_\lambda(\Vert \beta_j \Vert_{1})=\lambda \int_{0}^{\Vert\beta_j\Vert_{1}}(1-\frac{x}{\gamma\lambda})_+dx,\gamma>0$, which is a typical example of implementing variable selection with non-convex penalty terms. \citet{HB2016} proposed Mnet based on MCP combined with Ridge regression. In addition, \citet{YY2021} proposed the SACE and GSACE, which are defined respectively as follows,
 \begin{eqnarray}\label{eq03}
 	p_\lambda(\Vert \beta_j \Vert_{1})=\frac{1}{2}\beta_j^2-d(\hat{\beta}_j^{lasso}) \beta_j+\lambda \Vert \beta_j \Vert_{1}, 
 \end{eqnarray} 
 and
 \begin{eqnarray}\label{eq04}
 	p_\lambda(\Vert \beta_j \Vert_{1})=\frac{1}{2}\beta_j^2-d(\hat{\beta}_j^{mcp}) \beta_j+\lambda \int_{0}^{\Vert\beta_j\Vert_{1}}(1-\frac{x}{\gamma\lambda})_+dx,\gamma>0.
 \end{eqnarray}

 For the Elastic Net (Enet) proposed by \citet{ZH2005}, which combines ridge regression and Lasso, the penalty  is given by:
 
 \begin{eqnarray}
 	p_\lambda(\Vert \beta_j \Vert_{1})=\lambda\gamma \Vert \beta_j \Vert_{1} + \lambda(1-\gamma) \beta_j^2, \quad 0 \leq \gamma < 1.
 \end{eqnarray}
 
 \citet{JY2010} showed that Enet exhibits variable selection consistency under the non-representability condition. However, due to the form of the ridge shrinkage factor, Enet applies the same penalty intensity to all regression coefficients. This characteristic could be improved, as the coefficients of important variables with larger eigenvalues should not be excessively shrunk to prevent information loss. Therefore, optimizing the penalty strategy to impose differentiated shrinkage on variables of varying importance can enhance model stability and interpretability.

 This study introduces the Single-Parametric Principal Component Selection Operator (SPPCSO), a novel method for variable selection in high-dimensional correlated data. By integrating principal component analysis with single-parametric estimation and \( L_1 \) regularization, SPPCSO applies an adaptively smoothed adjustment to the shrinkage factor, reducing shrinkage for important variables while increasing it for less relevant ones. This balances sparsity and information retention, improving model interpretability and predictive performance. 
 
 Theoretically, we show that SPPCSO achieves a smaller estimation error bound than existing methods, explaining its superior stability in numerical simulations. Moreover, its mathematical similarity to Elastic Net allows it to accommodate group effect structures, improving robustness in handling highly correlated variables.
 
 Theoretically, we demonstrate that SPPCSO satisfies estimation error bounds and variable selection consistency under certain conditions. This consistency ensures that the estimator correctly identifies all important variables and effectively excludes uncorrelated variables as the sample size increases, which is crucial for high-dimensional statistical analysis. These properties not only enhance the accuracy and stability of the model but also improve the reliability and interpretability of the results, providing strong theoretical support and practical value for variable selection in high-dimensional data.
 
 The rest of the paper is organized as follows: Section \ref{sec2} defines SPPCSO and discusses the theoretical and algorithmic advantages of SPPCSO. Section \ref{sec3} presents the statistical properties of the SPPCSO estimator and gives the proof. Simulation results comparing the proposed methods and others are presented in Section \ref{sec4}. In Section \ref{sec5}, rat gene expression data are used to show the performance of the proposed method. A conclusion is given in Section \ref{sec6}.

\section{Method}\label{sec2}

In this paper, $p$ and $q$ are allowed to grow with $n$. The data and coefficients are allowed to change as $n$ grows. For notational simplicity, we do not index them with $n$; meanwhile, the response and the predictors are assumed to be standardized: 1) $y$ is assumed to be centered at $\mathbf{0}$ to avoid the need for an intercept. 2) $X$ is assumed be standardized so $X^{T}_{j}X_{j}/n=1$ , for $j=1,2,\cdots,p$. 
\subsection{ Single-parametric principal component selection operator }

Assume $S=\left\lbrace j\subset{1,2,\cdots,p}:\beta^{*}_{j}\neq0\right\rbrace$, where $\beta^{*}$ represents the true value of the coefficients of the variables and $\vert S \vert=q$, we have $X_{S}^{T}X_{S}=UD U^{T}$, where $U$ is the orthogonal matrix with $q$ rows and $q$ columns, and its column vector is the normalized eigenvector of $X_{S}^{T}X_{S}$; $D=diag(d_{1},\cdots,d_{q})$, $d_{i}$ is the $i$-th eigenvalue of $X_{S}^{T}X_{S}$.  

The single-parametric principal components regression estimator is defined as follows:
\begin{eqnarray}\label{eq02}
	\tilde{\beta}_{S}=UAU^{T}\hat{\beta}_{S}^{ols},
\end{eqnarray}
where $A=diag(\frac{d_{1}-1+\theta}{d_{1}},\cdots,\frac{d_{r}-1+\theta}{d_{r}},\theta d_{r+1},\cdots,\theta d_{q})$   and the diagonal elements of A are the compression factors for this estimate, $r$ satisfies $d_{r}\geq1\geq d_{r+1}$ and $\theta\subset(d_{q},1)$, $\hat{\beta}_{S}^{ols}=(X_{S}^{T}X_{S})^{-1}X_{S}^{T}y$ is the least square estimation.

From (\ref{eq02}), it can be inferred that:
\begin{eqnarray*}
	\tilde{\beta_{S}}=(X_{S}^{T}X_{S}+P)^{-1}X_{S}^{T}y,
\end{eqnarray*}
where
\begin{eqnarray*}
	P=UKU^{T},
\end{eqnarray*}
\begin{eqnarray*}
	K=diag(\frac{d_{1}(1-\theta)}{d_{1}+\theta-1},\cdots,\frac{d_{r}(1-\theta)}{d_{r}+\theta-1},\frac{1}{\theta}-d_{r+1},\cdots,\frac{1}{\theta}-d_{q}),
\end{eqnarray*}
let $Z_{S}=\sqrt{K}U^{T}$, for the given parameters $\lambda>0$ and $\theta\subset(d_{q},1)$, combined with the $L_{1}$ penalty, we define the Single-Parametric Principal Components Selection Operator(SPPCSO):
\begin{eqnarray}\label{eq11}
	\hat{\beta}:=\underset{\beta}{\text{argmin}}\{{\frac{1}{2n}\Vert y-X\beta\Vert_{2}^{2}+\frac{1}{2n}\Vert Z\beta\Vert_{2}^{2}+\lambda\Vert\beta\Vert_{1}}\},
\end{eqnarray}

\noindent
where $Z=\left(\begin{matrix}
	Z_{S}, 0\\
	0,  0 
\end{matrix}\right)_{p\times p}$.
\begin{proposition}
	We define an artificial data set $(y^{*},X^{*})$ by:
	\begin{eqnarray*}
		X^* =\left(\begin{matrix}
			X\\ Z
		\end{matrix}\right), \quad
		y^* =\left(\begin{matrix}
			y\\ 0
		\end{matrix}\right),
	\end{eqnarray*}
\end{proposition}
\noindent
thus the SPPCSO can be written as:
\begin{eqnarray*}
	\hat{\beta}:=\underset{\beta}{\text{argmin}}{\frac{1}{2n}\Vert y^{*}-X^{*}\beta\Vert_{2}^{2}+\lambda\Vert\beta\Vert_{1}}.
\end{eqnarray*}

With the above transformation, we transform the original optimization problem into a Lasso-type optimization problem, which makes SPPCSO enjoy the computational advantages of the Lasso,i.e., it guarantees the sparsity of the variables while making the model very interpretable.
\subsection{Theoretical justification of SPPCSO}
Next, we will explore the theoretical advantages of the SPPCSO in terms of the advantages of single-parametric principal component regression estimator over the Ridge and Liu estimation.

Ridge estimation, proposed by \citet{HK1970}, is a shrinkage estimator designed to address the instability of least squares estimation when the design matrix is non-invertible. It is defined as follows:
\begin{align*}
	\hat{\beta}_{S}^{ridge}&=(X^{T}_{S}X_{S}+kI)^{-1}X^{T}_{S}X_{S}\hat{\beta}_{S}^{ols}\\\nonumber
	&=(UDU^{T}+UkIU^{T})^{-1}UDU^{T}\hat{\beta}_{S}^{ols}\\\nonumber
	&=U(D+kI)^{-1}DU^{T}\hat{\beta}_{S}^{ols}\nonumber,
\end{align*}

Liu estimation is a proposed improvement of ridge estimation by introducing  parameter $h$ to make the estimation more flexible and adaptable. It is defined as:
\begin{align*}
\hat{\beta}_{S}^{liu}&=(X^{T}_{S}X_{S}+I)^{-1}X^{T}_{S}y+(X^{T}_{S}X_{S}+I)^{-1}h\hat{\beta}_{S}^{ols}\\\nonumber
&=(X^{T}_{S}X_{S}+I)^{-1}X^{T}_{S}X_{S}\hat{\beta}^{ols}+(X^{T}_{S}X_{S}+I)^{-1}h\hat{\beta}_{S}^{ols}\\\nonumber
&=(UDU^{T}+UIU^{T})^{-1}(UDU^{T}+UhIU^{T})\hat{\beta}_{S}^{ols}\\\nonumber
&=U(D+I)^{-1}(D+hI)U^{T}\hat{\beta}_{S}^{ols}\nonumber,
\end{align*}

We summarize the shrinkage factors of Liu estimation, Ridge estimation, and Single-Parametric Principal Component Regression estimation (SPPCR estimation) in the following table. Where $d_{i}$ is the eigenvalue of $X^{T}X$.

\begin{table}[htbp]
	\centering
	\renewcommand{\arraystretch}{2.8} 
	\caption{Shrinkage factors of different estimation methods}\label{table1}
	\label{tab:shrinkage_factors}
	\begin{tabular}{lc}
		\toprule
		\textbf{Estimation Method} & \textbf{Shrinkage Factor} \\
		\midrule
		Ridge estimation & $\frac{d_{i}}{d_{i} + k} = 1 - \frac{k}{d_{i} + k}$ \\
		Liu estimation & $\frac{d_{i} + h}{d_{i} + 1}$ \\
		SPPCR estimation& 
		$\begin{cases} 
			\frac{d_{i} - 1 + \theta}{d_{i}}, & d_{i} \geq 1 \\[5pt] 
			\theta d_{i}, & d_{i} < 1
		\end{cases}$ \\
		\bottomrule
	\end{tabular}
\end{table}

By comparing the shrinkage factor  the three estimation methods under different parameter settings, the following observations can be made:

\begin{itemize}
	\item \textbf{Ridge Estimation}: For variables with smaller eigenvalues, the shrinkage factor remains relatively large, and for larger eigenvalues, its decline is less significant. This results in insufficient screening of unimportant features and limited retention of important ones.
	
	\item \textbf{Liu Estimation}: The shrinkage factor increases more slowly than that of Ridge Estimation for smaller eigenvalues, indicating that Liu Estimation applies stronger compression to variables with smaller eigenvalues.
	
	\item \textbf{SPPCR Estimation}: This method exhibits a distinct behavior. For smaller eigenvalues, especially when the parameter $\theta$ is small, the shrinkage factor increases rapidly, leading to stronger compression of variables with smaller eigenvalues, thus making the screening of unimportant features more effective. When the eigenvalues exceed 1, the shrinkage factor increases more gradually and asymptotically approaches 1, indicating better retention of information for variables with larger eigenvalues.
\end{itemize}

Therefore, the SPPCSO proposed after combining $L_{1}$ regularization has the following theoretical advantages:
\begin{itemize}
	\item[(1)] Since SACE and GSACE are regularized versions of Ridge Estimation, and Mnet and Enet are closely related to Ridge Estimation, these methods are less flexible than SPPCSO in handling the shrinkage of variable coefficients with different importance levels, leading to reduced stability.
	
	\item[(2)] SCAD/MCP and other non-convex penalty functions enhance sparsity and reduce estimation bias. However, they may suffer from computational instability, sensitivity to initial estimates, and difficulty in effectively handling group effects in highly correlated variable environments.
	
	\item[(3)] Lasso tends to select only one variable from a group of highly correlated variables, which results in excessive shrinkage of coefficients and a loss of valuable variable information.
\end{itemize}

\section{  Error bound  analysis and  variable selection  consistency  }\label{sec3}
The next results are concerned with the error bounds and the consistency of variable selection of SPPCSO. For the simplicity of the proof, we set $\hat\beta^{0}$ as the Lasso estimator with the same tuning parameter $\lambda$. Assume $S=\left\lbrace j\subset{1,2,\cdots,p}:\beta^{*}_{j}\neq0\right\rbrace,\vert S \vert=q$ and $\hat{S}=\left\lbrace j\subset{1,2,\cdots,p}:\hat{\beta}_{j}\neq0\right\rbrace$, where  $\beta^{*}$ represents the true value of $\beta$. We begin by stating and proving the following Lemma:

\begin{lemma}\label{le1}
	Assume $\varepsilon_{i},i=1,2,\cdots,n$ are Gaussian random variables with mean $0$ and variance $\sigma^{2}$. Let $\Lambda_{max}\left\lbrace \cdot\right\rbrace $ be the maximum eigenvalue of the matrix, condition on $\frac{\Vert X^{T}\varepsilon\Vert_{\infty}}{n}\leqslant\frac{1}{2}\lambda$ and $\Lambda_{max}\left\lbrace \frac{1}{n}Z^{T}Z\right\rbrace\Vert\beta^{*}\Vert_{\infty}\leqslant\frac{1}{4}\lambda $,where $\lambda=3\sigma\sqrt{\frac{2\log p}{n}}$ and $p=O(exp(n^{c_{1}})),0<c_{1}<1$. We have $\Vert\hat{\beta}_{S^{c}}\Vert_{1}\leqslant 7\Vert\hat{\beta}_{S}-\beta^{*}_{S}\Vert_{1}$ .
\end{lemma}
\noindent\textbf{Proof of Lemma\ref{le1}}
Since $\hat{\beta}$ is the solution of
\begin{eqnarray*}
	\hat{\beta}:=\underset{\beta}{\text{argmin}}{\frac{1}{2n}\Vert y-X\beta\Vert_{2}^{2}+\frac{1}{2n}\Vert Z\beta\Vert_{2}^{2}+\lambda\Vert\beta\Vert_{1}},
\end{eqnarray*}
let the $Z$ corresponding to the final estimation obtained by SPPCSO be $\hat{Z}$, then there holds the following inequality:
\begin{eqnarray*}
	\frac{1}{2n}\Vert y-X\hat{\beta}\Vert_{2}^{2}+\frac{1}{2n}\hat{\beta}^{T}\hat{Z}^{T}\hat{Z}\hat{\beta}+\lambda\Vert\hat{\beta}\Vert_{1}\leqslant\frac{1}{2n}\Vert y-X\beta^{*}\Vert_{2}^{2}+\frac{1}{2n}(\beta^{*})^{T}Z^{T}Z\beta^{*}+\lambda\Vert \beta^{*}\Vert_{1},
\end{eqnarray*}
that's equivalent to
\begin{align*}
	&\frac{1}{2n}\Vert y-X\hat{\beta}\Vert_{2}^{2}+\frac{1}{2n}\hat{\beta}^{T}\hat{Z}^{T}\hat{Z}\hat{\beta}+\lambda\Vert\hat{\beta}\Vert_{1}-\frac{1}{2n}\Vert y-X\beta^{*}\Vert_{2}^{2}-\frac{1}{2n}(\beta^{*})^{T}Z^{T}Z\beta^{*}-\lambda\Vert \beta^{*}\Vert_{1}\nonumber \\
	&=G_{1}+G_{2}+G_{3}\nonumber \\
	&\leqslant0\nonumber ,
\end{align*}
\begin{eqnarray*}
	G_{1}=\frac{1}{2n}\Vert y-X\hat{\beta}\Vert_{2}^{2}-\frac{1}{2n}\Vert y-X\beta^{*}\Vert_{2}^{2},
\end{eqnarray*}
\begin{eqnarray*}
	G_{2}=\frac{1}{2n}\hat{\beta}^{T}\hat{Z}^{T}\hat{Z}\hat{\beta}-\frac{1}{2n}(\beta^{*})^{T}Z^{T}Z\beta^{*},
\end{eqnarray*}
\begin{eqnarray*}
	G_{3}=\lambda\Vert\hat{\beta}\Vert_{1}-\lambda\Vert \beta^{*}\Vert_{1},
\end{eqnarray*}
according to $y-X\beta^{*}=\varepsilon$, $G_{1}$ can be written as 
\begin{eqnarray*}
	G_{1}=\frac{1}{2n}(\beta^{*}-\hat{\beta})^{T}X^{T}X(\beta^{*}-\hat{\beta})+\frac{1}{n}\varepsilon^{T}X(\beta^{*}-\hat{\beta}),
\end{eqnarray*}
since Lasso has variable selection consistency, then $\hat{S}\xrightarrow{p} S$, $\hat{Z}\xrightarrow{p} Z$ , therefore holds  $G_{2}=\frac{1}{2n}(\hat{\beta}-\beta^{*})^{T}Z^{T}Z(\hat{\beta}+\beta^{*})$  with high probability.
Since $\Vert\beta^{*}_{S^{c}}\Vert_{1}=0 $, the $G_{3}$ can be written as:
\begin{eqnarray*}
	G_{3}=\lambda(\Vert\hat{\beta}_{S}\Vert_{1}-\Vert\beta^{*}_{S}\Vert_{1}+\Vert\hat{\beta}_{S^{c}}\Vert_{1}),
\end{eqnarray*}
further the above inequality can be transformed into
\begin{align}\label{eq03}
	&\frac{1}{n}(\beta^{*}-\hat{\beta})^{T}(X^{T}X+Z^{T}Z)(\beta^{*}-\hat{\beta})+2\lambda\Vert\hat{\beta}_{S^{c}}\Vert_{1}\\
	&\leqslant\frac{2}{n}\varepsilon^{T}X(\hat{\beta}-\beta^{*})+\frac{2}{n}(\beta^{*}-\hat{\beta})Z^{T}Z\beta^{*}+2\lambda\Vert\hat{\beta}_{S}-\beta^{*}_{S}\Vert_{1}\nonumber.
\end{align}

For $P\left\lbrace \Vert\frac{1}{n}X^{T}\varepsilon\Vert_{\infty}>t\right\rbrace$ 
\begin{eqnarray*}
	P\left\lbrace \Vert\frac{1}{n}X^{T}\varepsilon\Vert_{\infty}>t\right\rbrace \leqslant p\cdot P\left\lbrace \Vert\frac{1}{n}X_{j}^{T}\varepsilon\Vert_{1}>t\right\rbrace ,
\end{eqnarray*}
condition on $\frac{1}{n}X_{j}^{T}\varepsilon\sim N(0,\frac{\sigma^{2}}{n})$  we have	
\begin{eqnarray*}
P\left\lbrace \Vert\frac{1}{n}X^{T}\varepsilon\Vert_{\infty}>t\right\rbrace\leqslant2\exp(\log p-\frac{nt^{2}}{2\sigma^{2}}).
\end{eqnarray*}

There holds $P\left\lbrace \Vert\frac{1}{n}X^{T}\varepsilon\Vert_{\infty}\leqslant t\right\rbrace\xrightarrow{n\rightarrow\infty}1$ when $t=O(\sigma\sqrt{\frac{2\log p}{n}})$,
thus $ \Vert\frac{1}{n}X^{T}\varepsilon\Vert_{\infty}\leqslant\frac{1}{2}\lambda$ holds with high probability when $\lambda=3\sigma\sqrt{\frac{2\log p}{n}}$.
Then the following inequality holds with high probability: 
\begin{eqnarray*}
	\frac{2}{n}\varepsilon^{T}X(\hat{\beta}-\beta^{*})\leqslant2\Vert\frac{1}{n}X^{T}\varepsilon\Vert_{\infty}\Vert\hat{\beta}-\beta^{*}\Vert_{1}\leqslant\lambda\Vert\hat{\beta}_{S}-\beta^{*}_{S}\Vert_{1}+\lambda\Vert\hat{\beta}_{S^{c}}\Vert_{1},
\end{eqnarray*}
condition on $\Lambda_{max}\left\lbrace \frac{1}{n}Z^{T}Z\right\rbrace\Vert\beta^{*}\Vert_{\infty}\leqslant \frac{1}{4}\lambda,x<1 $ ,we have
\begin{align*}
	&\frac{2}{n}(\beta^{*}-\hat{\beta})Z^{T}Z\beta^{*}\leqslant\Vert\frac{2}{n}Z^{T}Z\beta^{*}\Vert_{\infty}\Vert\beta^{*}-\hat{\beta}\Vert_{1}\\
	&\leqslant2\Lambda_{max}\left\lbrace \frac{1}{n}Z^{T}Z\right\rbrace\Vert\beta^{*}\Vert_{\infty}\Vert\beta^{*}-\hat{\beta}\Vert_{1}\\
	&\leqslant\frac{1}{2}\lambda\Vert\hat{\beta}_{S^{c}}\Vert_{1}+\frac{1}{2}\Vert\hat{\beta}_{S}-\beta^{*}_{S}\Vert_{1},
\end{align*}
according to above two inequalities, (\ref{eq03}) can be written by
\begin{eqnarray}\label{eq10}
	\frac{1}{n}(\beta^{*}-\hat{\beta})^{T}(X^{T}X+Z^{T}Z)(\beta^{*}-\hat{\beta})+\frac{1}{2}\lambda\Vert\hat{\beta}_{S^{c}}\Vert_{1}\leqslant\frac{7}{2}\lambda\Vert\hat{\beta}_{S}-\beta^{*}_{S}\Vert_{1}.
\end{eqnarray}

Here the following result holds
\begin{eqnarray*}
	\Vert\hat{\beta}_{S^{c}}\Vert_{1}\leqslant7\Vert\hat{\beta}_{S}-\beta^{*}_{S}\Vert_{1}.
\end{eqnarray*}

Then we have the following result.
\begin{theorem}\label{th1}
	Let $C=X^{T}X/n$,$\nu=\beta^{*}-\hat{\beta}$. Assume $C$ satisfies the Restricted Eigenvalue (RE) condition: with a positive constant $\kappa$ that $\nu^{T}C\nu\geq\kappa\Vert\nu\Vert^{2}_{2}$ for all $\nu\in c(O),c(O):=\{\nu\in R^{p}:\Vert\nu_{o^{c}}\Vert_{1}\leqslant3\Vert\nu_{o}\Vert_{1}\}$ where $O\in {\{1,2,\cdots,p\}}$. There exists a positive constant $K$ that the SPPCSO satisfies the bound $\Vert\beta^{*}-\hat{\beta}\Vert_{2}\leqslant K$.   
\end{theorem}
\noindent\textbf{Proof of Theorem \ref{th1}}
According to the Restricted Eigenvalue (RE) condition, here the following inequality holds
\begin{align*}
	\frac{1}{n}(\beta^{*}-\hat{\beta})^{T}(X^{T}X+Z^{T}Z)(\beta^{*}-\hat{\beta})+\lambda\Vert\hat{\beta}_{S^{c}}\Vert_{1}&\geq\frac{1}{n}(\beta^{*}-\hat{\beta})^{T}X^{T}X(\beta^{*}-\hat{\beta})\\
	&\geq\kappa\Vert\beta^{*}-\hat{\beta}\Vert_{2}^{2},
\end{align*}
by (\ref{eq10}) and $\Vert x\Vert_{1}\leqslant\Vert x\Vert_{2}\Vert x\Vert^{\frac{1}{2}}_{0}$ , we have:
\begin{align*}
	\kappa\Vert\beta^{*}-\hat{\beta}\Vert_{2}^{2}\leqslant\frac{7}{2}\lambda\Vert\hat{\beta}_{S}-\beta^{*}_{S}\Vert_{1}&\leqslant\frac{7}{2}\lambda\Vert\hat{\beta}_{S}-\beta^{*}_{S}\Vert_{2}\cdot\Vert\hat{\beta}_{S}-\beta^{*}_{S}\Vert_{0}^{\frac{1}{2}},
\end{align*}
hence
\begin{eqnarray*}
	\Vert\beta^{*}-\hat{\beta}\Vert_{2}\leqslant K\sqrt{\frac{qlogp}{n}}.
\end{eqnarray*}
where $K=\frac{21\sqrt{2}\sigma}{2\kappa}$.
  
  For the condition  $\Lambda_{max}\left\lbrace\frac{1}{n}Z^{T}Z\right\rbrace\Vert\beta^{*}\Vert_{\infty}\leqslant\frac{1}{4}\lambda $ in the Lemma\ref{le1}, we can easily know $\Lambda_{max}\left\lbrace\frac{1}{n}Z^{T}Z\right\rbrace=max\{\frac{d_{r}(1-\theta)}{d_{r}+\theta-1},\frac{1}{\theta}-d_{q}\}$ by (\ref{eq11}). Thus we have $\theta\geq f(n,\lambda,\Vert\beta^{*}\Vert_{\infty})$. In other words, we can adjust the parameters $\lambda$ and $\theta$ to ensure that this condition is satisfied. Therefore, the setting of this condition is completely reasonable and easy to achieve.

  Furthermore, regarding the final result of the estimation error bound, we obtain a smaller constant factor K compared to the estimation error bound of SACE. This implies that, under the same data conditions, our estimator consistently achieves a smaller estimation error.Then we prove variable selection consistency of  SPPCSO.

\begin{theorem}\label{th2}
	Let $C=\left(\begin{matrix}
		n^{-1}X_S^\prime X_S, n^{-1}X_S^\prime X_{S^c}\\
		n^{-1}X^\prime_{S^c} X_S, n^{-1}X_{S^c}^\prime X_{S^c} 
	\end{matrix}\right)=
	\left(\begin{matrix}
		C_{11}, C_{12}\\
		C_{21}, C_{22} 
	\end{matrix}\right)$, condition on following 4 conditions, SPPCSO has variable selection consistency:
	$P\{sgn(\hat{\beta})=sgn(\beta^{*})\}\rightarrow 1$ .
	
\end{theorem}
\begin{itemize}
	\item[(C1)]The random error vectors of the model are independently and identically distributed in the normal distribution $N(0,\sigma^{2})$.
	\item[(C2)]There exists positive constants $N_{1}$,$N_{2}$,$M_{1}$,$M_{2}$,$R_{1}$,$R_{2}$ such that $0<N^{1}<\Lambda_{min}(C_{11}+n^{-1}Z^{T}_{S}Z_{S})<N_{2}<\infty,0<M_{1}<\Lambda_{max}(C_{11})<M_{2}<\infty,0<R_{1}<\Lambda_{max}(Z^{T}_{S}Z_{S})<R_{2}<\infty$ holds.
	\item[(C3)]$b_{1}=min\{\vert\beta^{*}_{j}\vert,j\in S\}$.
	\item[(C4)]$q=O(exp(n^{c_{2}})),0<c_{2}<1$ and $\sqrt{log(q)}/\sqrt{n}b_{1}\rightarrow0$.
\end{itemize}
\noindent\textbf{Proof of Theorem \ref{th2}}
From the KKT condition we know that:
\begin{eqnarray}\label{eq04}
	\left\{\begin{array}{l}
		\frac{1}{n}X^{T}_{j}(y-X\hat{\beta})=\frac{1}{n}Z^{T}Z\hat{\beta}_{j}+\lambda sgn(\hat{\beta}_{j}),\hat{\beta}_{j}\neq0,\\
		\vert\frac{1}{n} X^{T}_{j}(y-X\hat{\beta})\vert<\lambda,\hat{\beta}_{j}=0.
	\end{array}\right.
\end{eqnarray}

To prove consistency in variable selection, we need only show that $(\hat{\beta}^{T}_{S},0^{T})^{T}$ satisfies \ref{eq04}. Set $\hat{\gamma}_{S}=(\lambda sgn(\hat{\beta}_{j}),j\in S)^{T}$, we have:
\begin{eqnarray}\label{eq05}
	\left\{\begin{array}{l}
		\frac{1}{n}X^{T}_{S}(y-X_{S}\hat{\beta}_{S})=\frac{1}{n}Z_{S}^{T}Z_{S}\hat{\beta}_{S}+\hat{\gamma}_{S},\\
		\vert\frac{1}{n} X^{T}_{j}(y-X_{S}\hat{\beta}_{S})\vert<\lambda,\forall j \in S^{c}.
	\end{array}\right.
\end{eqnarray}

Set $Z_{S}^{T}Z_{S}+X_{S}^{T}X_{S}=V_{S}$, $H_{S}=I_{q}-X_{S}V_{S}^{-1}X_{S}$, from \ref{eq05} and $y=X_{S}\beta^{*}_{S}+\varepsilon$ , therefore
\begin{eqnarray*}
	\hat{\beta}_{S}=\beta^{*}_{S}+V_{S}^{-1}(X^{T}_{S}\varepsilon-n\hat{\gamma}_{S}-Z_{S}^{T}Z_{S}\beta^{*}_{S}),
\end{eqnarray*}
\begin{eqnarray*}
	y-X_{S}\hat{\beta}_{S}=H_{S}\varepsilon+X_{S}V_{S}^{-1}(n\hat{\gamma}_{S}+Z_{S}^{T}Z_{S}\beta^{*}_{S}),
\end{eqnarray*}
we need to show that the following inequality holds with high probability: 
\begin{eqnarray}\label{eq06}
	\left\{\begin{array}{l}
		\vert\hat{\beta}_{j}-\beta^{*}_{j}\vert<\vert\beta^{*}_{j}\vert,\forall j\in S,\\
		\frac{1}{n}\vert X_{j}^{T}H_{S}\varepsilon+X_{j}^{T}X_{S}V_{S}^{-1}(n\hat{\gamma}_{S}+Z_{S}^{T}Z_{S}\beta^{*}_{S})\vert<\lambda,\forall j\in S^{c}.
	\end{array}\right.
\end{eqnarray}
Then we prove the following cross problem holds with high probability:
\begin{align}\label{eq07}
	\vert\hat{\beta}_{j}-\beta^{*}_{j}\vert&=V_{S}^{-1}(X^{T}_{S}\varepsilon-n\hat{\gamma}_{S}-Z_{S}^{T}Z_{S}\beta^{*}_{S})\\
	&\leqslant\vert V_{S}^{-1}X_{S}^{T}\varepsilon\vert+\vert V_{S}^{-1}n\hat{\gamma}_{S}\vert+\vert V_{S}^{-1}Z_{S}^{T}Z_{S}\beta^{*}_{S}\vert\nonumber\\
	&<\vert\beta^{*}_{j}\vert\nonumber,
\end{align}
and
\begin{align}\label{eq08}
	&\frac{1}{n}\vert X_{j}^{T}H_{S}\varepsilon+X_{j}^{T}X_{S}V_{S}^{-1}(n\hat{\gamma}_{S}+Z_{S}^{T}Z_{S}\beta^{*}_{S})\vert\\
	&\leqslant\frac{1}{n}\vert X_{j}^{T}H_{S}\varepsilon\vert+\vert X_{j}^{T}X_{S}V_{S}^{-1}\hat{\gamma}_{S}\vert+\frac{1}{n}\vert X_{j}^{T}X_{S}V_{S}^{-1}Z_{S}^{T}Z_{S}\beta^{*}_{S}\vert\nonumber\\
	&<\lambda\nonumber.
\end{align}

We can show that the probability of the complementary problems of \ref{eq07} and \ref{eq08} converges to 0, we have
\begin{align*}
	P\{sgn(\hat{\beta})\neq sgn(\beta^{*})\}&\leqslant P\{\vert e^{T}_{j}V_{S}^{-1}X_{S}^{T}\varepsilon\vert\geq\frac{1}{3}\vert\beta^{*}_{j}\vert, \text{for some } j\in S\}\\
	&+P\{\vert e^{T}_{j} V_{S}^{-1}n\hat{\gamma}_{S}\vert\geq\frac{1}{3}\vert\beta^{*}_{j}\vert, \text{for some } j\in S\}\\
	&+P\{\vert e^{T}_{j}V_{S}^{-1}Z_{S}^{T}Z_{S}\beta^{*}_{S}\vert\geq\frac{1}{3}\vert\beta^{*}_{j}\vert, \text{for some } j\in S\}\\
	&+P\{\frac{1}{n}\vert X_{j}^{T}H_{S}\varepsilon\vert\geq\frac{1}{3}\lambda, \text{for some } j\in S^{c}\}\\
	&+P\{\vert X_{j}^{T}X_{S}V_{S}^{-1}\hat{\gamma}_{S}\vert\geq\frac{1}{3}\lambda, \text{for some } j\in S^{c}\}\\
	&+P\{\frac{1}{n}\vert X_{j}^{T}X_{S}V_{S}^{-1}Z_{S}^{T}Z_{S}\beta^{*}_{S}\vert\geq\frac{1}{3}\lambda, \text{for some } j\in S^{c}\}\\
	&=P\{T_{1}\}+P\{T_{2}\}+P\{T_{3}\}+P\{T_{4}\}+P\{T_{5}\}+P\{T_{6}\}\longrightarrow0.
\end{align*}

By condition C(2) ,we have
\begin{eqnarray*}
	\Lambda_{max}(V_{S}^{-1})=\Lambda_{min}^{-1}(V_{S})=\frac{1}{n\Lambda_{min}(C_{11}+n^{-1}Z_{S}^{T}Z_{S})}<\frac{1}{nN_{1}},
\end{eqnarray*}
condition on $\Vert e_{j}^{T}V_{S}^{-1}X_{S}^{T}\Vert_{2}^{2}\leqslant\frac{M_{2}}{nN_{1}^{2}}$ and condition (C4) ,we have
\begin{eqnarray*}
	P\{T_{1}\}\leqslant2\exp(\log q-\frac{nb_{1}^{2}N_{1}^{2}}{18M_{2}\sigma^{2}})\longrightarrow0.
\end{eqnarray*}

Since $\vert e^{T}_{j} V_{S}^{-1}n\hat{\gamma}_{S}\vert\leqslant\frac{\vert e_{j}^{T}\hat{\gamma_{S}}\vert}{N_{1}}=\frac{\lambda}{N_{1}}$ and $var(\frac{\lambda}{N_{1}})=0$, we have:
\begin{eqnarray*}
	P\{T_{2}\}\leqslant2\exp(\log q-\frac{b_{1}^{2}}{18var(\frac{\lambda}{N_{1}})})\longrightarrow0.
\end{eqnarray*}

Similar to 	$P\{T_{2}\}$ and combined with condition (C2) , $P\{T_{3}\}\longrightarrow0$ is easily accessible. Due to $X_{S}(V_{S})^{-1}X_{S}^{T}=X_{S}X_{S}^{T}(X_{S}X_{S}^{T}+Z_{S}Z_{S}^{T})^{-1}$ and $X_{j}^{T}X_{j}/n=1$, we obtain $\Lambda_{max}(H_{S})<1$ and $\Vert X_{j}^{T}H_{S} \Vert_{2} <\sqrt{n}$, therefore 

\begin{eqnarray*}
	P\{T_{4}\}\leqslant2\exp(\log q-\frac{n\lambda^{2}}{18\sigma^{2}})\longrightarrow0.
\end{eqnarray*}

Since $\Vert X_{j}^{T}X_{S}V_{S}^{-1}\Vert_{2}^{2}\leqslant\frac{X_{j}^{T}X_{S}V_{S}^{-1}X_{S}^{T}X_{j}}{nN_{1}}\leqslant \frac{1}{N_{1}} $, therefore $P\{T_{5}\}\leqslant P\{\Vert \hat{\gamma_{S}}\Vert \geq\frac{\lambda N_{1}}{3}\}\longrightarrow0$.

\noindent

Condition on condition(C2), we obtain $\Vert X_{j}^{T}X_{S}V_{S}^{-1}Z_{S}^{T}Z_{S}\Vert_{2}^{2}\leqslant\frac{R_{2}^{2}X_{j}^{T}X_{S}V_{S}^{-1}X_{S}^{T}X_{j}}{nN_{1}}\leqslant\frac{R_{2}^{2}}{N_{1}}$, therefore
\begin{align*}
	&P\{T_{6}\}\leqslant P\{\frac{1}{n}\Vert X_{j}^{T}X_{S}V_{S}^{-1}Z_{S}^{T}Z_{S}\Vert_{2}\Vert\beta^{*}_{S}\Vert_{2}\geq\frac{1}{3}\lambda,\exists j\in S_{c}\}\nonumber\\
	&\leqslant P\{\Vert\beta^{*}_{S}\Vert_{2}\geq\frac{n\sqrt{N_{1}}\lambda}{3R_{2}},\exists j\in S_{c}\}\longrightarrow0.
\end{align*}

Hence $P\{T_{1}\}+P\{T_{2}\}+P\{T_{3}\}+P\{T_{4}\}+P\{T_{5}\}+P\{T_{6}\}\longrightarrow0$ and $P\{sgn(\hat{\beta})=sgn(\beta^{*})\}\longrightarrow1$ .

\section{Simulations}\label{sec4}
\subsection{The coordinate decent algorithm for the SPPCSO}

\begin{algorithm}
	\caption{The coordinate decent algorithm for the SPPCSO}
	\begin{algorithmic}[1]
			\Require Initial values $\hat{\beta}^{0}$, given parameters $\lambda$, $\theta$, tolerance $\epsilon=10^{-4}$.Define a set $\xi^{0}=\{j\in \{1,\cdots,p:\hat{\beta}_{j}^{0}\neq0\}\}$ and an artificial dataset.
		
		\begin{eqnarray*}
			X^* =\left(\begin{matrix}
				X\\ Z
			\end{matrix}\right), \quad
			y^* =\left(\begin{matrix}
				y\\ 0
			\end{matrix}\right).
		\end{eqnarray*}
		where $Z=\sqrt{K}U^{T}$ and $	K=diag(\frac{d_{1}(1-\theta)}{d_{1}+\theta-1},\cdots,\frac{d_{r}(1-\theta)}{d_{r}+\theta-1},\frac{1}{\theta}-d_{r+1},\cdots,\frac{1}{\theta}-d_{p})$ , U is the orthogonal matrix obtained by performing singular value decomposition on $X^{T}_{\xi^{0} }X_{\xi^{0}}$, $d_{i}$ is the $i$-th eigenvalue of $X^{T}_{\xi^{0} }X_{\xi^{0}}$.
		\Ensure Optimized coefficients $\hat{\beta}^{SPPCSO}$
		
		\State Initialize $\hat{\beta}^{0} = \hat{\beta}^{Lasso}$ and set iteration counter $k = 0$
		\Repeat
		\For{each coefficient $j = 1$ to $p$}
		\State Compute the partial residual $r_j =( X^{*}_j)^{T}(y^{*} - X^{*}_{-j} \hat{\beta}^{(k)}_{-j})$
		\State Update $\hat{\beta}^{k+1}_j=S(r_j, \lambda)$, where$S(r, \lambda)=sign(r)(\vert r \vert-\lambda)_{+}$
		\EndFor
		\State Check for convergence: \If{$\|\hat{\beta}^{(k+1)} - \hat{\beta}^{(k)}\| < \epsilon$ }
		\State $\hat{\beta}^{SPPCSO}=\hat{\beta}^{k+1}$
		\State Stop the algorithm 
		\EndIf
		\State Increment $k$
		\Until{convergence}
		\Return $\hat{\beta}^{SPPCSO}$
	\end{algorithmic}
\end{algorithm}
The algorithms for optimization-type problems are the coordinate descent algorithm, forward stepwise algorithm, least angle regression algorithm, and so on. In this paper, the coordinate descent algorithm is used to solve the SPPCSO estimator. Coordinate descent is a non-gradient optimization algorithm. The algorithm performs a one-dimensional search in one coordinate direction at the current point in each iteration to find the local minima of a function. The Picasso package is used in the program to implement the computation, the details of the algorithm used in this paper are as follows:

\begin{figure}[!htbp]
	\centering
	\includegraphics[width=0.78\textwidth]{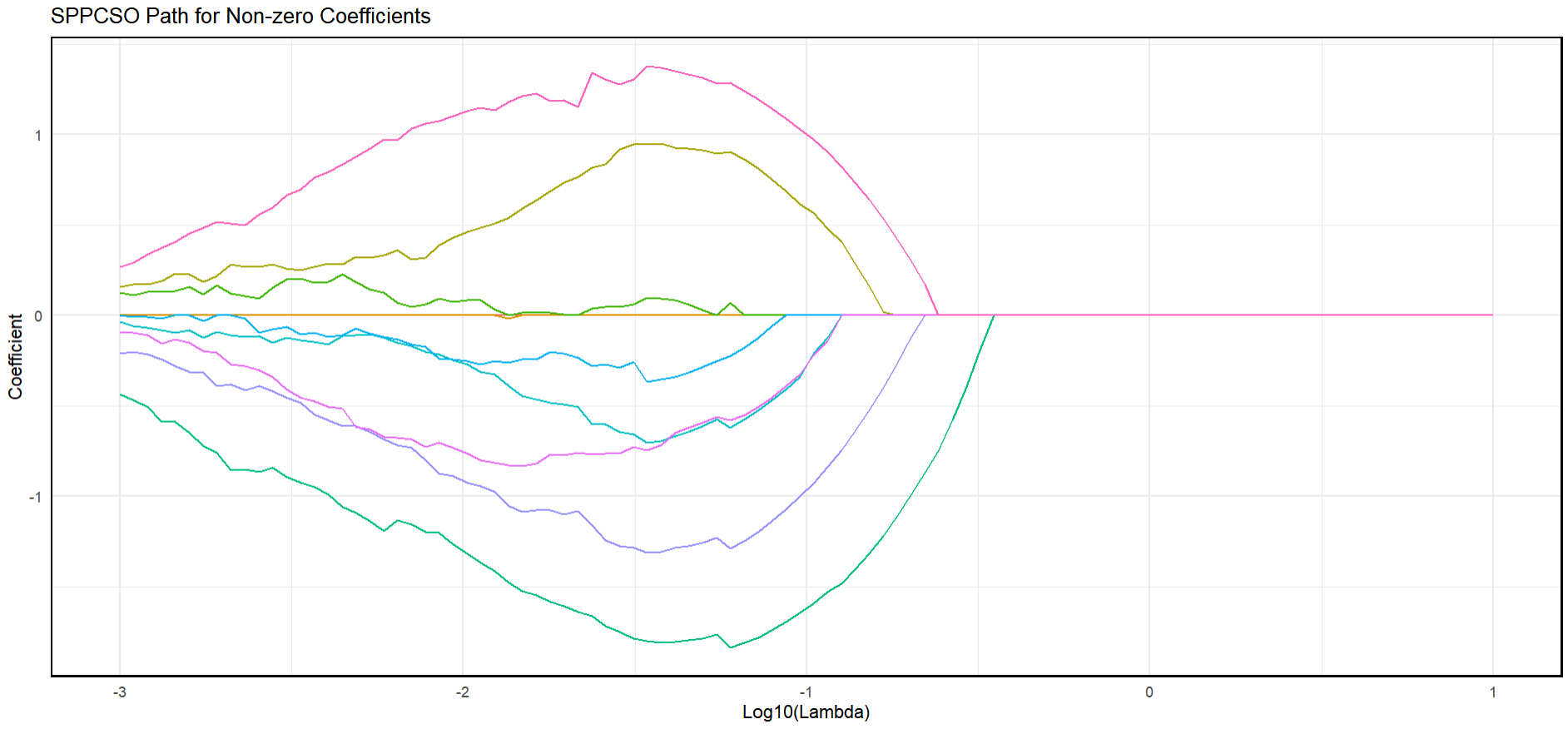}
	\caption{Solution path of the SPPCSO with respect to the parameter $\lambda$}
	\label{Fig1}
\end{figure}
In the process of solving the SPPCSO estimator, we choose Lasso estimation as the initial estimation because Lasso has variable selection consistency under irrepresentable conditions, which ensures that the final set of selected important variables converges to the true set of important variables with probability under the same parameter $\lambda$.

\begin{figure}[htbp]
	\centering
	\begin{minipage}{0.45\textwidth}
		\centering
		\includegraphics[width=\textwidth]{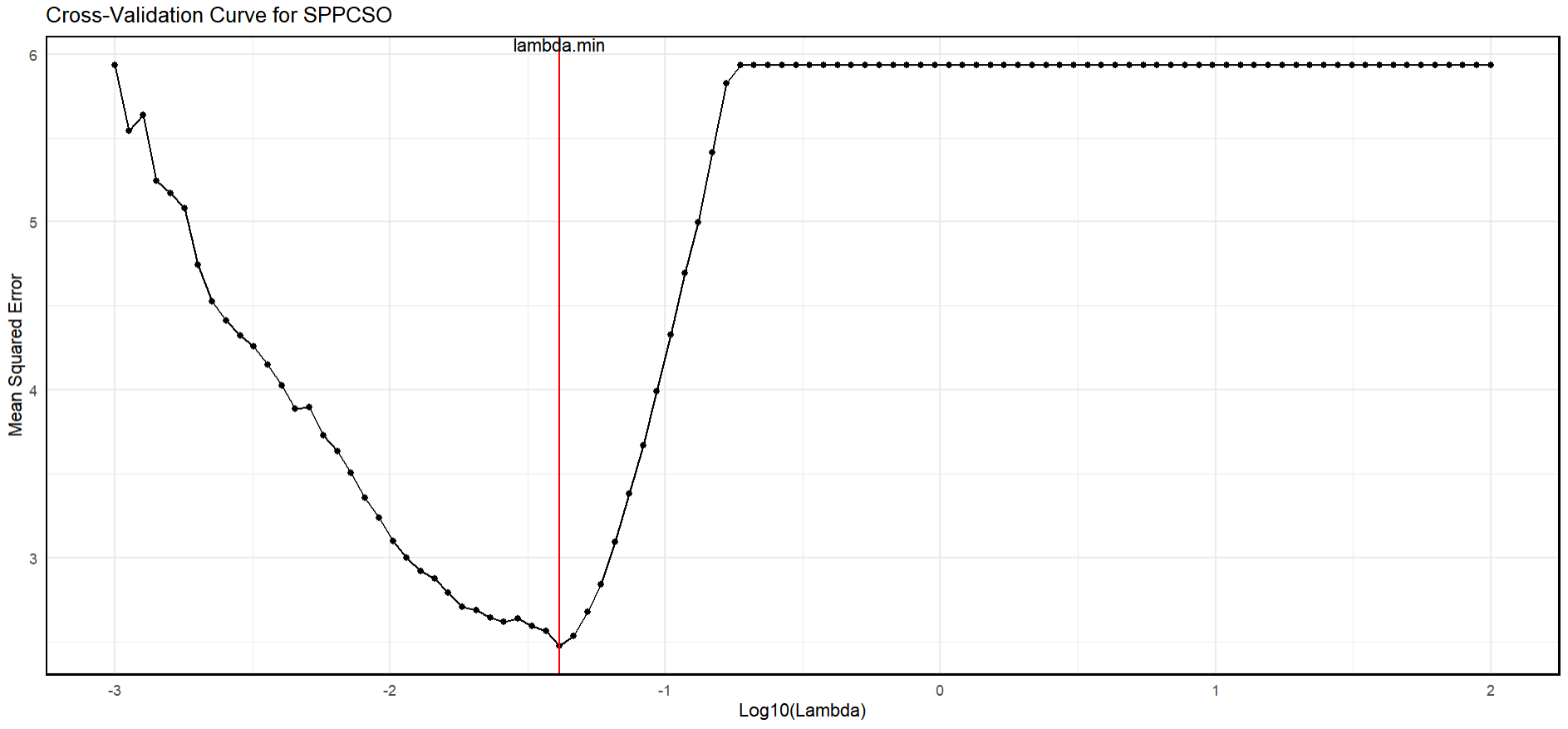}
		\caption{ Cross-validation curve with respect to the parameter $\lambda$ }
		\label{Fig2}
	\end{minipage}\hfill
	\begin{minipage}{0.45\textwidth}
		\centering
		\includegraphics[width=\textwidth]{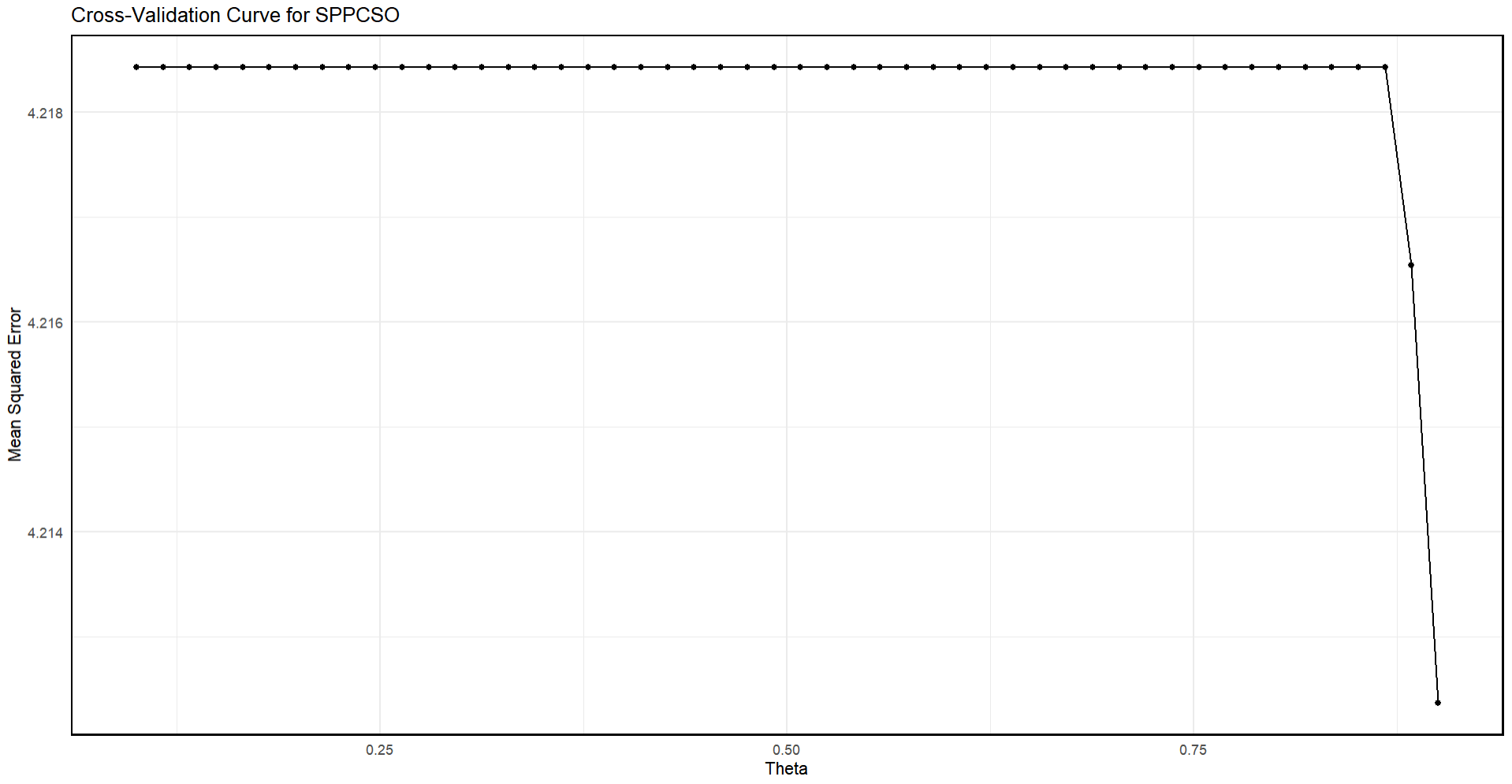}
		\caption{Cross-validation curve with respect to the parameter $\theta$ }
		\label{Fig3}
	\end{minipage}
\end{figure}
We employ a 5-fold cross-validation approach to select $\theta$ and $\lambda$. We define a grid for $\theta$ with values ranging from 0 to 1 in 0.1 increments. For each combination of parameters, we train the model on the training data and calculate the mean square error (MSE) on the validation set across all folds. The combination yielding the lowest average MSE is deemed optimal.

The SPPCSO solution path diagram(Fig.\ref{Fig1}) illustrates how $\lambda$ influences non-zero coefficient estimates in high-dimensional data. At larger $\lambda$ values, strong regularization forces all coefficients toward zero, simplifying the model. As $\lambda$ is reduced, this regularization effect weakens, allowing some coefficients to revert to their actual values. Further reduction in $\lambda$ may lead to a brief increase and subsequent decrease in coefficients due to the model capturing noise or redundant information, particularly in high-dimensional settings. This suggests that larger $\lambda$ values effectively suppress noise and redundancy, whereas too small a $\lambda$ might cause the model to overfit unnecessary information. From the solution path diagram, the optimal $\lambda$ interval appears to be between 0.01 and 0.1.

The Fig.\ref{Fig2} displays the mean square error (MSE) across different $\lambda$ values. The optimal $\lambda$ ($\lambda.min$), where MSE is minimized, is approximately 0.04. This indicates the best regularization strength for balancing performance between training and validation sets. Similarly, Fig.\ref{Fig3} suggests that the optimal $\theta$ should be as close to 1 as possible, where the MSE reaches its minimum.

\subsection{Some numerical experiments}

In this section, we generate simulation data to assess the performance of the proposed estimator, especially in finite sample scenarios. We compare our estimator with others, including Lasso, MCP, SCAD, Enet, Mnet, SACE, and GCACE, using seven criteria: (1) $\Vert\hat{\beta}-\beta^{*}\Vert_{2}$ (estimation error), (2)standard deviation of estimation error, (3) $\frac{1}{n}\Vert X(\hat{\beta}-\beta^{*})\Vert_{2}^{2}$ (prediction error), (4) standard deviation of prediction error, (5)TPR, (6)TNR, and (7)TMR.

TPR, TNR, and TMR are defined as:
\begin{align}
	\text{TPR}&=\frac{1}{N}\sum_{i=1}^{N}\frac{\#\{\hat{\beta}_j\neq0,j\in \mathcal{S}\}^i}{q_n},\nonumber\\
	\text{TNR}&=\frac{1}{N}\sum_{i=1}^{N}\frac{\#\{\hat{\beta}_j=0,j\in \mathcal{S}^c\}^i}{p_n-q_n},\nonumber\\
	\text{TMR}&=\frac{1}{N}\sum_{i=1}^{N}I\{\hat{\beta}_\mathcal{S}\neq \mathbf{0} \text{ and } \hat{\beta}_{\mathcal{S}^c}=\mathbf{0}\}^i,\nonumber
\end{align}
where i denotes the i-th repetition of the experiment. For the above seven evaluation indicators, indicators (1) (2) (3) (4) are close to 0 the better, however, indicators (5) (6) (7) are as close to 1 the better.

Two cases of high-dimensional correlation data are considered to generate predictors:

\textit{Example 1} In this example, we consider the case where there is partial orthogonality between the design array vectors and generate data $X\sim N_{p}(0,\Sigma)$ and $\varepsilon\sim N(0,\sigma^{2})$ , where $\Sigma_{jk}=0.95^{|j-k|},j,k=1,2,\cdots,15$ and $\Sigma_{jk}=0.95^{|j-k|},j,k=16,17,\cdots,p$ we set $\beta^{*}$ as $\beta^{*}_{S}\sim Unif(2,3),S=\{1,2,3,\cdots,15\}$. Set $\sigma=0.5,1$ and  $2 $.

\textit{Example 2}  In this example, we consider the group effects structural model. First, we define three groups as follows:  $X_{j}=x_{k}+\epsilon_{jk},k=1,6,11,j=k,\cdots,k+4$, where $x_{k}\sim N(0,1)$ and $\epsilon_{jk}\sim N(0,0.01)$ . For the remaining $p-15$ dimensions, we introduce different correlation structures by setting$\rho=0.5,0.75$  and $ 0.95$ .We set $\beta^{*}$ as $\beta^{*}_{S}\sim Unif(2,3),S=\{1,2,3,\cdots,15\}$.

	\begin{table}[htbp]
		\centering
		\caption{Average estimation error and  standard deviation for Example 1} \label{tab2}
		\label{tab:estimation_error}
		\begin{tabular}{lcccccc}
			\toprule
			\multirow{2}{*}{\textbf{Method}} & \multicolumn{2}{c}{\textbf{$\sigma$=0.5}} & \multicolumn{2}{c}{\textbf{$\sigma$=1}} & \multicolumn{2}{c}{\textbf{$\sigma$=2}} \\
			\cmidrule(lr){2-3} \cmidrule(lr){4-5} \cmidrule(lr){6-7}
			& \textbf{Est Error} & \textbf{Std Dev} & \textbf{Estn Error} & \textbf{Std Dev} & \textbf{Est Error} & \textbf{Std Dev} \\
			\midrule
			LASSO  & 0.9571 & 0.2083 & 1.7573 & 0.4200 & 3.6713 & 0.8382 \\
			MCP    & 8.0887 & 2.2153 & 9.0925 & 3.0784 & 9.9405 & 1.7509 \\
			SCAD   & 7.8636 & 3.2870 & 8.3720 & 2.4588 & 9.3404 & 1.7447 \\
			Enet   & 0.7861 & 0.1697 & 1.2426 & 0.3160 & 1.9663 & 0.8021 \\
			Mnet   & 0.6252 & 0.1383 & 0.9063 & 0.1692 & 1.2092 & 0.3593 \\
			SACE   & 0.9564 & 0.2071 & 1.7587 & 0.4274 & 3.6606 & 0.8516 \\
			GSACE  & 7.1896 & 1.4645 & 7.7451 & 1.595A3 & 8.8941 & 1.3764 \\
			\textbf{SPPCSO} & 1.0098 & 0.1504 & 1.0472 & 0.1593 & \textbf{1.1677} & \textbf{0.2435} \\
			\bottomrule
		\end{tabular}
	\end{table}

	\begin{table}[htbp]
		\centering
		\caption{Average prediction error and  standard deviation for Example 1} \label{tab3}
		\label{tab:prediction_error}
		\begin{tabular}{lcccccc}
			\toprule
			\multirow{2}{*}{\textbf{Method}} & \multicolumn{2}{c}{\textbf{$\sigma$=0.5}} & \multicolumn{2}{c}{\textbf{$\sigma$=1}} & \multicolumn{2}{c}{\textbf{$\sigma$=2}} \\
			\cmidrule(lr){2-3} \cmidrule(lr){4-5} \cmidrule(lr){6-7}
			& \textbf{Pre Error} & \textbf{Std Dev} & \textbf{Pre Error} & \textbf{Std Dev} & \textbf{Pre Error} & \textbf{Std Dev} \\
			\midrule
			LASSO  & 0.3252 & 0.0605 & 1.2850 & 0.2187 & 5.2246 & 0.9739 \\
			MCP    & 4.4767 & 13.2624 & 8.2810 & 21.9644 & 9.1475 & 2.5962 \\
			SCAD   & 8.3170 & 26.1689 & 7.0328 & 24.4160 & 8.2789 & 2.3673 \\
			Enet   & 0.3208 & 0.0554 & 1.2479 & 0.2025 & 4.9183 & 0.9608 \\
			Mnet   & 0.2918 & 0.0462 & 1.1589 & 0.1976 & 4.6989 & 0.7843 \\
			SACE   & 0.3249 & 0.0583 & 1.2850 & 0.2180 & 5.2173 & 0.9611 \\
			GSACE  & 2.4808 & 1.1554 & 3.6209 & 1.4004 & 7.6697 & 1.5058 \\
			\textbf{SPPCSO} & 0.5207 & 0.1064 & 1.3214 & 0.2289 &\textbf{ 4.6958} &\textbf{ 0.8919} \\
			\bottomrule
		\end{tabular}
	\end{table}

	\begin{table}[htbp]
		\centering
		\caption{TPR, TNR, and TMR for Example 1}
		\label{tab4}
		\label{tab:tpr_tnr_tmr}
		\begin{tabular}{lccccccccc}
			\toprule
			\multirow{2}{*}{\textbf{Method}} & \multicolumn{3}{c}{\textbf{$\sigma$=0.5}} & \multicolumn{3}{c}{\textbf{$\sigma$=1}} & \multicolumn{3}{c}{\textbf{$\sigma$=2}} \\
			\cmidrule(lr){2-4} \cmidrule(lr){5-7} \cmidrule(lr){8-10}
			& \textbf{TPR} & \textbf{TNR} & \textbf{TMR} & \textbf{TPR} & \textbf{TNR} & \textbf{TMR} & \textbf{TPR} & \textbf{TNR} & \textbf{TMR} \\
			\midrule
			LASSO  & 1.000 & 0.982 & 0.138 & 1.000 & 0.987 & 0.138 & 0.989 & 0.988 & 0.138 \\
			MCP    & 0.653 & 0.997 & 0.000 & 0.619 & 0.992 & 0.000 & 0.538 & 0.996 & 0.000 \\
			SCAD   & 0.698 & 0.983 & 0.013 & 0.661 & 0.987 & 0.000 & 0.577 & 0.987 & 0.000 \\
			Enet   & 1.000 & 0.981 & 0.138 & 1.000 & 0.985 & 0.125 & 0.999 & 0.982 & 0.088 \\
			Mnet   & 1.000 & 0.998 & 0.488 & 1.000 & 0.997 & 0.525 & 0.998 & 0.994 & 0.200 \\
			SACE   & 1.000 & 0.984 & 0.113 & 1.000 & 0.986 & 0.113 & 0.989 & 0.985 & 0.038 \\
			GSACE  & 0.701 & 0.998 & 0.000 & 0.677 & 0.996 & 0.000 & 0.603 & 0.994 & 0.000 \\
			\textbf{SPPCSO} & \textbf{1.000} & \textbf{1.000} & \textbf{0.988} & \textbf{1.000} &\textbf{ 0.999} & \textbf{0.700} & \textbf{1.000} & \textbf{0.992} & \textbf{0.225} \\
			\bottomrule
		\end{tabular}
	\end{table}

In both of these illustrative instances, we maintain a constant setting with $n=200$ and $p=600$ to satisfy $p\gg n$. The experiments encompass the comprehensive consideration of all predictor correlations, adhering to both the RE condition and the non-representability condition. Each example is subjected to 100 repetitions, resulting in N = 100, and the outcomes are derived as the average of 100 simulations.

In Example 1, we construct a dataset with a partially orthogonal structure and evaluate the performance of SPPCSO in high-dimensional sparse modeling tasks under three different noise levels, $\sigma=0.5,1,2$. Tables \ref{tab2}, \ref{tab3}, and \ref{tab4} present the estimation error, prediction error, and key variable selection metrics (TNR, TPR, TMR), respectively. 

It can be observed that under different noise intensities, SPPCSO consistently maintains lower estimation errors (Table \ref{tab2}) and prediction errors (Table \ref{tab3}), with significantly lower standard deviations compared to other methods, indicating the robust generalization ability of SPPCSO. Moreover, from the results of TNR, TPR, and TMR, SPPCSO demonstrates outstanding performance in variable selection. Even at a high noise level ($\sigma=2$), it can still correctly identify relevant variables, with its TMR values significantly higher than those of any other method. This simulation illustrates that SPPCSO exhibits excellent stability under partially orthogonal structured data. Even in the presence of high noise, it can still achieve reliable estimation.

\begin{table}[htbp]
	\centering
	\caption{Average estimation error and standard deviation for Example 2}
	\label{tab5}
	\label{tab:estimation_error_example2}
	\begin{tabular}{lcccccc}
		\toprule
		\multirow{2}{*}{\textbf{Method}} & \multicolumn{2}{c}{\textbf{$\rho$=0.5}} & \multicolumn{2}{c}{\textbf{$\rho$=0.75}} & \multicolumn{2}{c}{\textbf{$\rho$=0.95}} \\
		\cmidrule(lr){2-3} \cmidrule(lr){4-5} \cmidrule(lr){6-7}
		& \textbf{Est Error} & \textbf{Std Dev} & \textbf{Est Error} & \textbf{Std Dev} & \textbf{Est Error} & \textbf{Std Dev} \\
		\midrule
		LASSO  & 4.4559  & 1.3309  & 4.5233  & 1.0556  & 4.0679  & 0.9334 \\
		MCP    & 19.1823 & 0.7085  & 19.3239 & 0.6469  & 19.2281 & 0.6575 \\
		SCAD   & 19.1818 & 0.7072  & 19.3295 & 0.6493  & 19.2287 & 0.6628 \\
		Enet   & 2.2640  & 0.6605  & 2.5113  & 0.8459  & 2.1949  & 0.8089 \\
		Mnet   & 1.4280  & 1.1497  & 1.6669  & 1.4552  & 1.3523  & 1.0139 \\
		SACE   & 4.2984  & 0.8282  & 4.3517  & 0.9619  & 3.9473  & 0.8972 \\
		GSACE  & 19.1822 & 0.7009  & 19.2595 & 0.6774  & 19.1145 & 0.9082 \\
		\textbf{SPPCSO} & 	\textbf{1.2182}  &	\textbf{ 0.4172}  & \textbf{1.1597}  & 	\textbf{0.3494}  & 	\textbf{1.1147}  & 	\textbf{0.2718} \\
		\bottomrule
	\end{tabular}
\end{table}

\begin{table}[htbp]
	\centering
	\caption{Average prediction error and standard deviation for Example 2}
	\label{tab6}
	\label{tab:prediction_error_example2}
	\begin{tabular}{lcccccc}
		\toprule
		\multirow{2}{*}{\textbf{Method}} & \multicolumn{2}{c}{\textbf{$\rho$=0.5}} & \multicolumn{2}{c}{\textbf{$\rho$=0.75}} & \multicolumn{2}{c}{\textbf{$\rho$=0.95}} \\
		\cmidrule(lr){2-3} \cmidrule(lr){4-5} \cmidrule(lr){6-7}
		& \textbf{Pre Error} & \textbf{Std Dev} & \textbf{Pre Error} & \textbf{Std Dev} & \textbf{Pre Error} & \textbf{Std Dev} \\
		\midrule
		LASSO  & 1.5626  & 0.5447  & 1.5407  & 0.2755  & 1.3384  & 0.2503 \\
		MCP    & 4.8789  & 0.7955  & 4.9638  & 0.8179  & 4.8581  & 0.7325 \\
		SCAD   & 4.8857  & 0.7807  & 4.9612  & 0.8239  & 4.8676  & 0.7459 \\
		Enet   & 1.3901  & 0.2449  & 1.3847  & 0.2449  & 1.2293  & 0.2139 \\
		Mnet   & 1.3837  & 0.2623  & 1.3720  & 0.2376  & 1.2027  & 0.2174 \\
		SACE   & 1.5145  & 0.2823  & 1.5162  & 0.2892  & 1.3161  & 0.2347 \\
		GSACE  & 4.9215  & 0.7946  & 5.0052  & 0.7835  & 4.9417  & 0.7739 \\
		\textbf{SPPCSO} & 1.6554  & 0.4515  & 1.5743  & 0.3461  & 1.2788  & 0.2509 \\
		\bottomrule
	\end{tabular}
\end{table}

	\begin{table}[htbp]
		\centering
		\caption{TPR, TNR, and TMR for Example 2}
		\label{tab7}
		\label{tab:tpr_tnr_tmr_example2}
		\begin{tabular}{lccccccccc}
			\toprule
			\multirow{2}{*}{\textbf{Method}} & \multicolumn{3}{c}{\textbf{$\rho$=0.5}} & \multicolumn{3}{c}{\textbf{$\rho$=0.75}} & \multicolumn{3}{c}{\textbf{$\rho$=0.95}} \\
			\cmidrule(lr){2-4} \cmidrule(lr){5-7} \cmidrule(lr){8-10}
			& \textbf{TPR} & \textbf{TNR} & \textbf{TMR} & \textbf{TPR} & \textbf{TNR} & \textbf{TMR} & \textbf{TPR} & \textbf{TNR} & \textbf{TMR} \\
			\midrule
			LASSO  & 0.983  & 0.956  & 0.013  & 0.988  & 0.964  & 0.000  & 0.991  & 0.970  & 0.013 \\
			MCP    & 0.200  & 0.999  & 0.000  & 0.200  & 0.998  & 0.000  & 0.200  & 0.998  & 0.000 \\
			SCAD   & 0.200  & 0.994  & 0.000  & 0.200  & 0.994  & 0.000  & 0.200  & 0.995  & 0.000 \\
			Enet   & 1.000  & 0.962  & 0.000  & 0.999  & 0.966  & 0.000  & 0.999  & 0.970  & 0.013 \\
			Mnet   & 0.981  & 0.986  & 0.025  & 0.968  & 0.988  & 0.038  & 0.983  & 0.991  & 0.038 \\
			SACE   & 0.988  & 0.956  & 0.013  & 0.987  & 0.964  & 0.000  & 0.991  & 0.972  & 0.000 \\
			GSACE  & 0.200  & 0.996  & 0.000  & 0.201  & 0.996  & 0.000  & 0.204  & 0.996  & 0.000 \\
			\textbf{SPPCSO} & \textbf{1.000} & \textbf{0.997} & \textbf{0.363} & \textbf{1.000} & \textbf{0.997} & \textbf{0.350} & \textbf{1.000} & \textbf{0.993} & \textbf{0.138} \\
			\bottomrule
		\end{tabular}
	\end{table}

\begin{figure}[!htbp]
	\centering
	\includegraphics[width=1\textwidth]{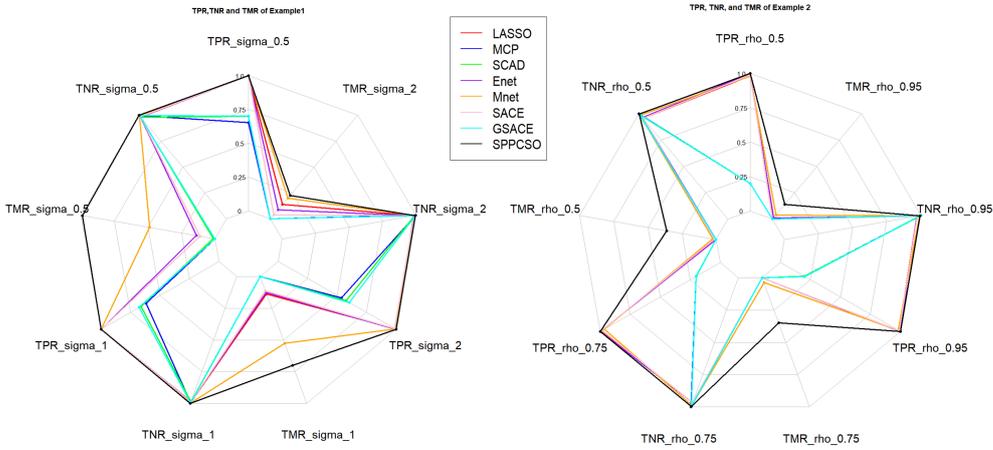}
	\caption{Performance comparison of methods. The image on the left is the radar chart of TPR, TNR, and TMR values for Example 1, while the one on the right corresponds to Example 2.}
	\label{exampel2}
\end{figure}

In Example 2, we further tested the adaptability of SPPCSO to the structure of group effects data by simulating the ($\rho=0.5,0.75,0.95$) group of noise variables set in the group effects structure data. Tables \ref{tab5}, \ref{tab6}, and \ref{tab7} provide a detailed summary of the experimental results. 

In terms of estimation error and its standard deviation, SPPCSO consistently outperforms all other methods at different levels of noise-variable correlation. Even in the case of a high correlation of $\rho=0.95$, its estimation error is still significantly lower than other methods (1.1147 vs. Lasso 4.0679, MCP 19.2281). This suggests that SPPCSO can still accurately distinguish between signal and noise variables even in a high covariance environment, rather than selecting the wrong variable by correlation interference. 

From the perspective of variable selection, even under high correlation ($\rho=0.95$), SPPCSO still maintains the highest TPR (1.000) and relatively high TMR (0.138), which ensures the high precision of variable selection. In contrast, nonconvex methods such as MCP and SCAD suffer from a severe lack of TMR (0.000) under such conditions.

This phenomenon can be attributed to the fact that SPPCSO combines principal component regression (PCR) with the $L_1$ penalty. This combination allows it to strike a balance between sparsity and information retention and thus is more adaptable to the group effects structure than traditional methods based on $L_1/L_2$ penalties (e.g., Enet, Mnet). This simulation further confirms that our model has a significant advantage in highly correlated data structures, with a better feature screening mechanism that effectively discards redundant variables. In the case of ultra-high dimensional data, the selection of key features is still stable.

				\section{Empirical analysis}\label{sec5}
				To illustrate the predictive advantages of SPPCSO in practical applications, we apply the SPPCSO to 120 samples of rat gene expression data reported by \cite{Seta2006}. This data includes gene expression values for 31,042 probes. The primary objective of this analysis was to examine how the expression of TRIM32 (probe 1289163 at), a gene known to cause inherited diseases of the human retina is dependent on the expression of other genes. Probes that were not expressed in the eye or had insufficient change were excluded from the set of 31,042 probes following the methods of \cite{HMZ2008}, \cite{LNP2014}, and \cite{MLT2017}: removing each probe with a maximum expression value in 120 rats that were less than the 25th percentile of the entire set of expression values and selecting probes that had at least a two-fold change in the expression level in the 120 rats of the probes that had at least a two-fold change in expression level in the 120 rats. There were 18986 probes remaining after this process. However, for ultra-high dimensional data, the estimates obtained using the variable selection method are less precise and computationally expensive, and to reduce the computational cost, the 3000 genes with the highest variance in expression values were selected after the above process. For comparison, we randomly selected 60 samples as the training set and re-selected 60 data as the test set. 
				
				The mean absolute prediction error (MAPE) was used to evaluate the prediction effect of different methods. 
				\begin{eqnarray*}
					\mathrm{MAPE}=\frac{1}{60}\sum_{a=1}^{100}\frac{1}{n}\sum_{i=1}^{n}\vert\hat{y}_i^a-y_i^a\vert,
				\end{eqnarray*}
				where $a$ denotes the result of the ath sample. We recorded the MAPE values and the number of nonzeroes (NNZ) in the parameter estimates for the training and test sets. In addition, we draw boxplots of the results of 100 repetitions of the experiment and visualize the standard deviation of the MAPE values and NNZ values of the test set by the width of the boxplots. Both the MAPE values and the NNZ values, and their standard deviations, can be used to measure the complexity of the model and the stability of the parameter estimates.
				\begin{table}[htbp] 
					\tabcolsep 0pt \vspace*{-12pt}
					\def\temptablewidth{1.0\textwidth}
					\centering
					\caption{MAPE and NNZ values for the training and test sets of rat genetic data} \label{tab8}
					\begin{tabular*}{\temptablewidth}{@{\extracolsep{\fill}} ccccccccc}
							\toprule
							Method	&Lasso &MCP &SCAD &Enet &Mnet &SACE &GCACE &SPPCSO\\
							\midrule
							MAPE(train)	&0.0385	&0.0534	&0.0502	& 0.0225& 0.0309& 0.0193& 0.0245&0.0282	\\
							MAPE(test)	&0.0908	&0.0936	&0.0884	&0.0855	&0.0944	&0.0841	&0.1063	&\textbf{0.0803}\\
							NNZ	&32.60	&25.40	&19.28	&805.64	& 254.94&82.90	& 135.38&72.44	\\
							\botrule
						\end{tabular*}%
					\end{table}
				
					\begin{figure}[htbp]
						\centering
						\begin{minipage}{0.45\textwidth}
							\centering
							\includegraphics[width=\textwidth]{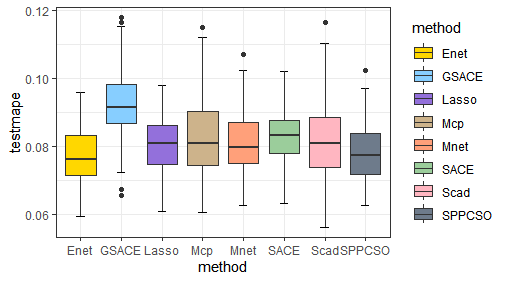}
							\caption{Box line plot of MAPE values }
							\label{fig:esti_error_example2}
						\end{minipage}\hfill
						\begin{minipage}{0.45\textwidth}
							\centering
							\includegraphics[width=\textwidth]{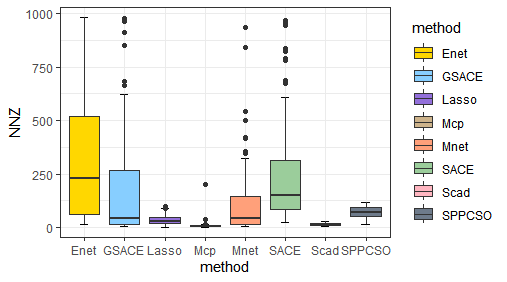}
							\caption{Box line plot of NNZ values }
							\label{fig:pre_error_example2}
						\end{minipage}
					\end{figure}
					
					From Table \ref{tab5} we find that SPPCSO has the smallest MAPE value in the test set, which indicates that in SPPCSO performs the best in terms of test error compared to other methods. In terms of the number of non-zero coefficients selected, SCAD selects the least number of variables, followed by MCP, Lasso, and SPPCSO. Although SCAD and MCP have better sparsity in selecting non-zero coefficients than SPPCSO, the prediction error is significantly larger than that of SPPCSO.  Lasso selects fewer variables than SPPCSO. However, Lasso, as a variable selection method, overcompresses the coefficients and tends to select only one of the highly correlated variables, which leads to the fact that despite its good sparsity, Lasso may estimate the coefficients of some of the variables that are important in their own right as zero, making the model lack of interpretability. 
					
					Therefore, from the table, SPPCSO chooses the sparser variables while ensuring the minimum testing error. From the box-and-line plot, we can see that SPPCSO has good stability of variable selection in repeated experiments. Comprehensively, SPPCSO has the best prediction performance and stability of variable selection in practical applications.

					\section{Conclusion}\label{sec6}
					
				In this paper, we introduced the Single-Parametric Principal Component Selection Operator (SPPCSO), a novel penalized estimation method that integrates the single-parametric principal component regression estimator with $L_{1}$ regularization. By adaptively adjusting shrinkage, SPPCSO effectively balances variable selection and information retention, improving model stability and reducing estimation variance.
				
				We established its theoretical properties, proving that SPPCSO achieves a smaller estimation error bound and satisfies variable selection consistency. To validate its performance, we conducted extensive numerical experiments under varying noise levels and correlation structures. The results demonstrated that SPPCSO consistently outperforms existing methods, maintaining robust and accurate estimation even in high-noise and highly correlated settings. Furthermore, real-data analysis in gene expression studies confirmed its ability to identify disease-related genes, highlighting its practical significance in high-dimensional applications.
				
				Overall, SPPCSO provides a powerful and flexible solution for high-dimensional variable selection, particularly in scenarios with strong correlations. Future research could explore extensions incorporating non-convex penalties or applications to structured sparsity problems, further enhancing its adaptability and effectiveness.
					
					\section{Acknowledgments}\label{sec7}
					This work is supported by the National Natural Science Foundation of
					China [Grant No. 12371281].The authors would like to express their sincere gratitude to the Associate Editor and the referees for their invaluable comments, which significantly contributed to enhancing the paper.

\end{document}